\documentclass{ecai}
\usepackage{times}
\usepackage{graphicx}
\usepackage{latexsym}

\usepackage[ruled, vlined]{algorithm2e}
\usepackage{amsfonts,bm}
\usepackage[hidelinks]{hyperref}
\usepackage{mathtools}
\usepackage{microtype}
\usepackage{nicefrac} 
\usepackage{times}
\usepackage{soul}
\usepackage{url}
\usepackage{floatrow}
\newfloatcommand{capbtabbox}{table}[][\FBwidth]
\usepackage{pdfpages}
\usepackage{floatrow}
\newcommand{\ignore}[1]{}
\DeclareMathOperator*{\argmax}{arg\,max}

\newcommand{\absl}[1]{\left \lvert #1 \right \rvert}
\newcommand{\bigo}[1]{\mathcal{O}\left ( #1 \right )}

\newcommand{\fantom}{\textsc{Fantom}\xspace}
\newcommand{\kgreedy}{\textsc{$\lambda$-Greedy}\xspace}

\newcommand{\opt}{\textsc{opt}\xspace}

\newcommand{\fantomalg}{\textsc{Fantom}}
\newcommand{\dgreedy}{\textsc{$\lambda$-dGreedy}\xspace}

\newtheorem{definition}[theorem]{Definition}
\newtheorem{lemma}[theorem]{Lemma}
\newtheorem{problem}[theorem]{Problem}
\newtheorem{reduction}[theorem]{Reduction}
\newenvironment{proof}{\paragraph{Proof:}}{\hfill$\square$\newline}

\usepackage[textsize=tiny]{todonotes}

\begin{document}

\title{Non-Monotone Submodular Maximization with Multiple Knapsacks in Static and Dynamic Settings.}

\author{Vanja Dosko\v{c}\institute{Hasso Plattner Institute, email: vanja.doskoc@hpi.de} \and Tobias Friedrich\institute{Hasso Plattner Institute, email: friedrich@hpi.de} \and  Andreas G\"{o}bel\institute{Hasso Plattner Institute, email: andreas.goebel@hpi.de} \and Aneta Neumann\institute{The University of Adelaide, email: aneta.neumann@adelaide.edu.au}\\ \and  Frank Neumann\institute{The University of Adelaide, email: frank.neumann@adelaide.edu.au} \and Francesco Quinzan\institute{Hasso Plattner Institute, email: francesco.quinzan@hpi.de} }

\maketitle
\bibliographystyle{ecai.bst}

\begin{abstract}
We study the problem of maximizing a non-monotone submodular function under multiple knapsack constraints. We propose a simple discrete greedy algorithm to approach this problem, and prove that it yields strong approximation guarantees for functions with bounded curvature. In contrast to other heuristics, this does not require problem relaxation to continuous domains and it maintains a constant-factor approximation guarantee in the problem size. In the case of a single knapsack, our analysis suggests that the standard greedy can be used in non-monotone settings.

Additionally, we study this problem in a dynamic setting, in which knapsacks change during the optimization process. We modify our greedy algorithm to avoid a complete restart at each constraint update. This modification retains the approximation guarantees of the static case. 

We evaluate our results experimentally on a video summarization and sensor placement task. We show that our proposed algorithm competes with the state-of-the-art in static settings. Furthermore, we show that in dynamic settings with tight computational time budget, our modified greedy yields significant improvements over starting the greedy from scratch, in terms of the solution quality achieved. 
\end{abstract}

\section{INTRODUCTION}
\label{sec:intro}

Many artificial intelligence and machine learning tasks can be naturally approached by maximizing submodular objectives. Examples include subset selection~\cite{DasK18}, document summarization~\cite{Lin:2010}, video summarization \cite{MirzasoleimanBK16} and action recognition \cite{Zheng:2014}. Submodular functions are set functions that yield a diminishing return property: it is more helpful to add an element to a smaller collection than to add it to a larger one. This property fully characterizes the notion of submodularity.

Practical applications often require additional side constraints on the solution space, determined by possible feasibility conditions. These constraints can be complex \cite{Lin:2010,MirzasoleimanJ018,YuXC18}. For instance, when performing video summarization tasks, we might want to select frames that fulfill costs constraints based on qualitative factors, such as resolution and luminance.

In this paper, we study general multiple knapsack constraints. Given a set of solutions, a $k$-knapsack constraint consists of $k$ linear cost functions $c_i$ on the solution space and corresponding weights $W_i$. A solution is then feasible if the corresponding costs do not exceed the weights. In this paper, we study the problem of maximizing a submodular function under a $k$-knapsack constraint.

Sometimes, real-world optimization problems involve dynamic and stochastic constraints~\cite{Chiong2012}. For instance, resources and costs can exhibit slight frequent changes, leading to changes of the underlying space of feasible solutions. Various optimization problems have been studied under dynamically changing constraints, i.e., facility location problems~\cite{Jena:2016}, target tracking~\cite{FazlyabPPR16}, and other submodular maximization problems for machine learning~\cite{Borodin:2017}. Motivated by these applications, we also study the problem of maximizing a submodular function under a $k$-knapsack constraint, when the set of feasible solutions changes online.
\paragraph{Literature Overview.} Khuller, Moss and Naor~\cite{DBLP:journals/ipl/KhullerMN99} show that a simple greedy algorithm achieves a $1/2(1-1/e)$-approximation guarantee, when maximizing a modular function with a single knapsack constraint. They also propose a modified greedy algorithm that achieves a $(1-1/e)$-approximation. Sviridenko~\cite{DBLP:journals/orl/Sviridenko04} shows that this modified greedy algorithm yields a $(1-1/e)$-approximation guarantee for monotone submodular functions under a single knapsack constraint. Its run time is $O(n^5)$ function evaluations. The optimization of monotone submodular functions under a given chance constraint by an adaptation of these simple greedy algorithm has recently been investigated by Doerr et al.~\cite{DoerrAAAI20}.

Lee et al.~\cite{DBLP:conf/stoc/LeeMNS09} give a $(1/5-\varepsilon)$-approximation local search algorithm, for maximizing a non-monotone submodular function under multiple knapsack constraints. Its run time is polynomial in the problem size and exponential in the number of constraints. Fadaei, Fazli and Safari~\cite{FADAEI2011447} propose an algorithm that achieves a ($1/4-\varepsilon$)-approximation algorithm for non-monotone functions. This algorithm requires to compute fractional solutions of a continuous extension of the value oracle function $f$. Chekuri, Vondr{\'{a}}k and Zenklusen~\cite{DBLP:journals/siamcomp/ChekuriVZ14} improve the approximation ratio to $0.325-\varepsilon$, in the case of $k = \bigo{1}$ knapsacks. Kulik, Schachnai and Tamir~\cite{Kulik13} give a $(1-1/e-\varepsilon)$-approximation algorithm when $f$ is monotone and a $(1/e-\varepsilon)$-approximation algorithm when the function is non-monotone. Again, their method uses continuous relaxations of the discrete setting. \fantom{} is a popular algorithm for non-monotone submodular maximization \cite{DBLP:conf/icml/MirzasoleimanBK16}. It can handle intersections of a variety of constraints. In the case of multiple knapsack constraints, it achieves a $1/((1 + \varepsilon)(10 + 4 k))$-approximation in $\bigo{n^2 \log(n) / \varepsilon}$ run time.

Submodular optimization problems with dynamic cost constraints, including knapsack constraints, are investigated in \cite{DBLP:conf/ppsn/Roostapour0N18,DBLP:journals/corr/abs-1811-07806}. Rostapoor et al.~\cite{DBLP:journals/corr/abs-1811-07806}  show that a Pareto optimization approach can implicitly deal with dynamically changing constraint bounds, whereas a simple adaptive greedy algorithm fails.

\paragraph{Our Contribution.} Many of the aforementioned algorithmic results, despite having polynomial run time, seem impractical for large input applications. Following the analysis outlined in \cite{DBLP:journals/dam/ConfortiC84,DBLP:journals/corr/abs-1811-05351,DBLP:journals/ipl/KhullerMN99}, we propose a simple and practical discrete algorithm to maximize a submodular function under multiple knapsack constraints. This algorithm, which we call the \kgreedy, achieves a $(1 - e^{ -1/ \lambda})/(3 \max(1, \alpha))$-approximation guarantee on this problem, with $\alpha$ expressing the curvature of $f$, and $\lambda \in [1, k]$ a constant. It requires at most $\mathcal{O}(n^{\max(k/\lambda, 2)})$ function evaluations. To our knowledge this is the first algorithm yielding a trade-off between run-time and approximation ratio.\\
We also propose a robust variation of our \kgreedy, which we call \dgreedy{}, to handle dynamic changes in the feasibility region of the solution space. We show that this algorithm maintains a $(1 - e^{ -1/ \lambda})/(3 \max(1, \alpha))$-approximation without having to do a complete restart of the greedy sequence.

We demonstrate experimentally that our algorithms yield good performance in practise, with two real-world scenarios. First, we consider a video summarization task, which consists of selecting representative frames of a given video \cite{MirzasoleimanJ018,DBLP:conf/icml/MirzasoleimanBK16}. We also consider a sensor placement problem, that asks to select informative thermal stations over a large territory \cite{DBLP:journals/jmlr/KrauseSG08}.\\
Our experiments indicate that the \kgreedy{} yields superior performance to commonly used algorithms for the static video summarization problem. We then perform experiments in dynamic settings with both scenarios, to show that the robust variation yields significant improvement in practise.

The paper is structured as follows. In Section \ref{sec:problem_description} we introduce basic definitions and define the problem. In Section \ref{sec:algorithms} we define the algorithms. We present the theoretical analysis in Section \ref{sec:approx_ratios}, and the experimental framework in Section \ref{sec:experimental_framework}. The experimental results are discussed in Section \ref{sec:experiments_static} and Section \ref{sec:experiments_dynamic}. We conclude in Section \ref{sec:conclusions}.

\section{PRELIMINARIES}
\label{sec:problem_description}
\paragraph{Submodularity and Curvature}
In this paper, we consider problems with an \emph{oracle function} that outputs the quality of given solution. We measure performance in terms of calls to this function, since in many practical applications they are difficult to evaluate. We assume that value oracle functions are submodular, as in the following definition.
\begin{definition}[Submodularity]
\label{def:submodular}
Given a finite set $V$, a set function $f\colon 2^V\rightarrow \mathbb{R}$ is submodular if  for all $S,\Omega\subseteq V$ we have that $f(S) + f(\Omega) \geq f(S \cup \Omega) + f(S \cap \Omega)$.
\end{definition}

For the equivalent intuitive definition described informally in the introduction see~\cite{DBLP:journals/mp/NemhauserWF78}. 

For any submodular function $f\colon 2^V \rightarrow \mathbb{R}$ and sets $S, \Omega \subseteq V$, we define the \emph{marginal value} of $\Omega$ with respect to $S$ as $f_S(\Omega) = f(S\cup \Omega) - f(S)$. Note that, if $f$ only attains non-negative values, it holds that $f(\Omega) \geq f_S(\Omega)$ for all $S, \Omega \subseteq V$. 

Our approximation guarantees use the notion of \emph{curvature}, a parameter that bounds the maximum rate with which a submodular function changes. We say that a submodular function $f\colon 2^V\rightarrow \mathbb{R}_{\geq 0}$ has curvature $\alpha$ if the value $f(S\cup e) - f(S)$ does not change by a factor larger than $1 - \alpha$ when varying $S$, for all $e \in V \setminus S$. This parameter was first introduced by \cite{DBLP:journals/dam/ConfortiC84} and later revisited in \cite{DBLP:conf/icml/BianB0T17}. We use the following definition of curvature, which slightly generalizes that proposed in Friedrich et al. \cite{DBLP:journals/corr/abs-1811-05351}. 
\begin{definition}[Curvature]
\label{def:generalized_curvature}
Let $f\colon 2^V \rightarrow \mathbb{R}_{\geq 0}$ be a submodular set function. The curvature is the smallest scalar $\alpha$ such that
\[
f_\omega ((S \cup \Omega )\setminus \{\omega\}) \geq (1 - \alpha)f_\omega (S \setminus \{\omega\}),
\]
for all $S, \Omega \subseteq V$ and $\omega \in S\setminus \Omega$.
\end{definition}
Note that $\alpha \leq 1$ always holds and that all monotone submodular functions have curvature always bounded as $0 \leq \alpha \leq 1$. It follows that all submodular functions with negative curvature are non-monotone.

\paragraph{Problem Description}
The problem of maximizing a submodular function under multiple knapsack constraints can be formalized as follows.
\begin{problem}
\label{problem:static}
Let $f\colon 2^V \rightarrow \mathbb{R}_{\geq 0}$ be a submodular function.\footnote{We assume that $f(\emptyset) = 0$, and that $f$ is non-constant.} Consider linear cost functions $c_i\colon 2^V \rightarrow \mathbb{R}_{\geq 0}$,\footnote{We assume that $\max_j c_j(e) > 0$ for all $e \in V$.} and corresponding weights $W_i$, for all $i \in [k]$. We search for a set $\opt \subseteq V$, such that $\opt \in \argmax_{S \subseteq V} \left \{ f(S )\colon c_i(S) \leq W_i, \forall i\in [k] \right \}$.
\end{problem}
In this setting, one has $k$ knapsacks and wishes to find an optimal set of items such that its total cost, expressed by the functions $c_i$, does not violate the capacity of each knapsack. Note that the same set might have different costs for different knapsacks.\\
We denote with $(c, W)$ the constraint requirements $c_i(S) \leq W_i$ for all $S \subseteq V$, for all $i \in [k]$. For a ground set $V$ with a fixed ordering on the points and a cost function $c_i$, where $i\in[k]$, let $\tau_i$ be a permutation on $V$ where $c_i(e_{\tau_i(1)}) \geq \dots \geq c_i(e_{\tau_i(n)})$. We define the value $\chi(c, W)$ as 
\[
\chi(c, W) = \min_{i \in [k]} \argmax_{j \in [\absl{V}]} \left \{ \sum_{\ell = 1}^j c_i(e_{\tau_i(\ell )}) \leq W_i \right \}.
\]
Note that the value $\chi(c, W)$ is such that each set $U \subseteq V$ with cardinality $\absl{U} \leq \chi (c, W)$ is feasible for all constraints in $(c, W)$.\footnote{We remark that for our purposes, the value $\chi(c, W)$ could be defined directly as the value where each set $U \subseteq V$ with cardinality $\absl{U} \leq \chi (c, W)$ is feasible under the constraints in $(c, W)$. However, this value is in general NP-hard to compute \cite{CHO20072456,641542}.}\\
We observe that in the case of a single knapsack, if $c_1(S) = \absl{S}$ for all $S \subseteq V$, then Problem \ref{problem:static} consists of maximizing a submodular function under a cardinality constraint, which is known to be NP-hard. 

In our analysis we always assume that the following reduction holds.
\begin{reduction}
\label{reduction:static}
For Problem \ref{problem:static} we may assume that there exists a point $e^* \in V$ such that $f(S \cup e^*) = f(S)$ for all $S \subseteq V$, and $c_i(e^*) = W_i$ for all $i \in [k]$. Furthermore, we may assume that $c_i(e) \leq W_i$ for all $e \in V$, for all $i \in [m]$.
\end{reduction}
If the conditions of Reduction~\ref{reduction:static} do not hold, one can remove all points $e \in V$ that violate one of the constraints and add a point $e^*$ without altering the function $f$.
Intuitively, Reduction \ref{reduction:static} requires that each singleton, except for one, is feasible for all knapsack constraints. This ensures that $\argmax_{e \in V}f(e)$ is always feasible in all constraints, since $f(e^*) = 0$, and the optimum solution consists of at least one point. Furthermore, the point $e^* \in V$ ensures that the solution quality never decreases throughout a greedy optimization process, until a non-feasible solution is reached.

Additionally, we study a dynamic setting of Problem \ref{problem:static}, in which weights $W_i$ are repeatedly updated throughout the optimization process, while the corresponding cost functions $c_i$ remain unchanged. In this setting, we assume that an algorithm queries a function to retrieve the weights $W_i$ which are, sometimes, updated online. We assume that weights changes occur independently of the optimization process and algorithmic operations. Furthermore, we assume that Reduction \ref{reduction:static} holds for each dynamic update.
\section{ALGORITHMS}
\label{sec:algorithms}
\begin{algorithm}[t]
	\caption{The \kgreedy{} algorithm.}
	\label{alg:static}
		\textbf{input}: submodular function $f$, $k$-knapsacks $(c, W)$ and parameter $\lambda$ \;
	\textbf{output}: an approximate feasible global maximum of $f$\;
	$\mathcal{V}\gets \{e \in V \colon c_j(e) \leq \lambda W_j/k \ \forall j \in [k]\}$\;
	$S \gets \mathcal{V}$\; 
	$\sigma \gets \emptyset $\;
	$v^* \gets \argmax_{e\in V}f(e)$\;
 	\While{$S \neq \emptyset$}{
 	        let $e\in S$ maximizing $f_{\sigma}(e)/\max_j c_j(e)$\;
            $S \gets S \setminus e$\;
            \lIf{$c_j(\sigma \cup e) \leq W_j, \forall j \in [k]$}{$\sigma \gets \sigma \cup e$}
   	}
   	\textbf{return} $\argmax_{\{U \subseteq V\setminus \mathcal{V} \colon c_i(U) \leq W_i \ \forall i \}} \{f(\sigma), f(v^*), f(U) \}$\;
\end{algorithm}
\begin{algorithm}[t]
	\caption{The \dgreedy{} algorithm.}
	\label{alg:dynamic}
		\textbf{input}: submodular function $f$, $k$-knapsacks $(c, W)$ with dynamic weights and parameter $\lambda$\;
	\textbf{output}: an approximate feasible global maximum of $f$\;
	evaluate $f$ over all sets of cardinality at most $k/\lambda$\;
	$\mathcal{V}\gets \{e \in V \colon c_j(e) \leq \lambda W_j/k \ \forall j \in [k]\}$\;
	$S \gets \mathcal{V}$\; 
	$\sigma \gets \emptyset $\;
	$v^* \gets \argmax_{e\in V}f(e)$\;
 	\While{$S \neq \emptyset$}{            
 	        $\ $\\
            // greedy rule 	\\        
            let $e\in S$ maximizing $f_{\sigma}(e)/\max_j c_j(e)$\;
            $S \gets S \setminus e$\;
            \lIf{$c_i(\sigma \cup e) \leq W_i, \forall i \in [k]$}{$\sigma \gets \sigma \cup v$}
            $\ $\\
            // update rule\\
            \If{new weights $\{W_i'\}$ are given}
            {
                $\mathcal{V}'\gets \{e \in V \colon c_j(e) \leq \lambda W'_j/k \ \forall j \in [k]\}$\;
                \While{$\absl{\sigma} > \min \{\chi(c, W), \chi (c, W')\} \lor \sigma \not \subseteq \mathcal{V} \cap \mathcal{V}'$
                }
                {
                let $e \in \sigma$ be the last point added to $\sigma$\;
                $\sigma \gets \sigma \setminus e$\;
                }
                $\mathcal{V} \gets \mathcal{V}'$\;
	            $S \gets \mathcal{V} \setminus \sigma$\; 
            }
   	}
   	\textbf{return} $\argmax_{\{U \subseteq V\setminus \mathcal{V} \colon c_i(U) \leq W \ \forall i \}} \{f(\sigma), f(v^*), f(U) \}$\;
\end{algorithm}
We approach Problem \ref{problem:static} with a discrete algorithm based on a greedy technique, commonly used to maximize a submodular function under a single knapsack constraint (see \cite{DBLP:journals/ipl/KhullerMN99,DBLP:conf/aaai/ZhangV16}). Given the parameter value $\lambda\in[k]$, our algorithm defines the following partition of the objective space:
\begin{enumerate}
    \item[$\bullet$] the set $\mathcal{V}$ containing all $e \in V$ such that $c_j \leq \lambda W_j/k$ for all $j \in [k]$;
    \item[$\bullet$] the complement $V \setminus \mathcal{V}$ containing all $e \in \mathcal{V}$ such that $c_j(e) > \lambda W_j/k$ for all $j \in [k]$.
\end{enumerate}
The $\kgreedy{}$ optimizes $f$ over the set $\mathcal{V}$, with a greedy update that depends on all cost functions $c_j$. After finding a greedy approximate solution $\sigma$, the $\kgreedy{}$ finds the optimum $\tau$ among feasible subsets of $V\setminus \mathcal{V}$. This step can be performed with a deterministic search over all possible solutions, since the space $V\setminus \mathcal{V}$ always has bounded size. The $\kgreedy{}$ outputs the set with highest $f$-value among $\sigma, \tau$ or the maximum among the singletons.

From the statement of Theorem \ref{theorem:static} we observe that the parameter $\lambda$ sets a trade-off between solution quality and run time. For small $\lambda $, Algorithm \ref{alg:static} yields better approximation guarantee and worse run time, than for large $\lambda$. This is due to the fact that the size of $\mathcal{V}\setminus V$ depends on this parameter. In practise, the parameter $\lambda$ allows to find the right trade-off between solution quality and run time, depending on available resources. Note that in the case of a single knapsack constraint, for $\lambda = k$ the \kgreedy{} is equivalent to the greedy algorithm studied in \cite{DBLP:journals/ipl/KhullerMN99}.
 
We modify the \kgreedy{} to handle dynamic constraints where weights change overtime. This algorithm, which we refer to as the \dgreedy{}, is presented in Algorithm \ref{alg:dynamic}. It consists of two subroutines, which we call the \emph{greedy rule} and the \emph{update rule}.
The greedy rule of the \dgreedy{} uses the same greedy update as the \kgreedy{} does: At each step, find a point $v \in V$ that maximizes the marginal gain over maximum cost, and add $v$ to the current solution, if the resulting set is feasible in all knapsacks.
The update rule allows to handle possible changes to the weights, even when the greedy procedure has not terminated, without having to completely restart the algorithm. 

Following the notation of Algorithm \ref{alg:dynamic}, if new weights $W'_1, \dots, W'_k$ are given, then the \dgreedy{} iteratively removes points from the current solution, until the resulting set yields $\sigma \leq \min \{\chi(c, W), \chi (c, W') \}$ and $\sigma \subseteq \mathcal{V} \cap \mathcal{V}'$. This is motivated by the following facts:
\begin{enumerate}
    \item every set $U \leq \min \{\chi(c, W), \chi (c, W') \}$ is feasible in both the old and the new constraints;
    \item every set $U \leq \min \{\chi(c, W), \chi (c, W') \}$ such that $U \subseteq \mathcal{V} \cap \mathcal{V}'$ obtained with a greedy procedure yields the same approximation guarantee in both constraints;
    \item every set $U \subseteq \mathcal{V} \cap \mathcal{V}'$ is such that $c_i(e) \leq \lambda W'_i/k$ for all $i \in k$, for all $e \in U$.
\end{enumerate}
All three conditions are necessary to ensure that the approximation guarantee is maintained.

Note that the update rule in Algorithm \ref{alg:dynamic} does not backtrack the execution of the algorithm until the resulting solution is feasible in the new constraint, and then adds elements to the current solution. For instance, consider the following example, due to Roostapour et al. \cite{DBLP:journals/corr/abs-1811-07806}. We are given a set of $n + 1$ items $\{e_1, \dots, e_{n + 1}\}$ under a single knapsack $(c, W)$, with the cost function $c$ defined as
\[
c(e_i) =
\left \{
\begin{array}{ll}
1, & \mbox{if } 1 \leq i \leq n/2 \mbox{ or } i = n + 1;\\
2, & \mbox{otherwise};\\
\end{array}
\right .
\]
and $f$-values defined on the singleton as
\[
f(e_i) =
\left \{
\begin{array}{ll}
1/n, & \mbox{if } 1 \leq i \leq n/2;\\
1  , & \mbox{if } n/2 < i \leq n;\\
3,   & \mbox{if } i = n + 1;\\
\end{array}
\right .
\]
We define $f(U) = \sum_{e \in U} f(e)$, for all $U \subseteq \{e_1, \dots, e_{n + 1}\}$. Consider Algorithm $\ref{alg:dynamic}$ optimizing $f$ from scratch with $W = 2$. We only consider the case $\lambda = 1$, since we only consider a single knapsack constraint. Then both Algorithm \ref{alg:dynamic} and the \kgreedy{} return a set of the form $\{e_{n + 1}, e_j \}$ with $1 \leq j \leq n/2$. Suppose now that the weight dynamically changes to $W = 3$. Then backtracking the execution and adding points to the current solution would result into a solution of the form $\{e_{n + 1}, e_i, e_j\}$ with $1 \leq i,j \leq n/2$ with $f(\{e_{n + 1}, e_i, e_j\}) = 3 + 2/n$. However, in this case it holds $\min \{\chi(c, 2), \chi(c, 3)\} = 1$, since there exists a solution of cardinality $2$ that is not feasible in $(c, 2)$. Hence, Algorithm \ref{alg:dynamic} removes the point $e_j$ from the solution $\{e_{n + 1}, e_j \}$, before adding new elements to it. Afterwards, it adds a point $e_{j'}$ with $n/2 < {j'} \leq n$ to it, obtaining a solution such that $f(\{e_{n + 1}, e_{j'} \}) = 3 + 1 = 4$.\\
We remark that on this instance, the \kgreedy{} with a simple backtracking operator yields arbitrarily bad approximation guarantee, as discussed in \cite[Theroem 3]{DBLP:journals/corr/abs-1811-07806}. In contrast, Algorithm \ref{alg:dynamic} maintains the approximation guarantee on this instance (see Theorem \ref{theorem:dynamic}).

\section{APPROXIMATION GUARANTEES}
\label{sec:approx_ratios}

We prove that Algorithm \ref{alg:static} yields a strong approximation guarantee, when maximizing a submodular function under $k$ knapsack constraints in the static case. This part of the analysis does not consider dynamic weight updates. We use the notion of curvature as in Definition \ref{def:generalized_curvature}.
\begin{theorem}
\label{theorem:static}
Let $f$ be a submodular function with curvature $\alpha$, suppose that $k$ knapsacks are given. For all $\lambda \in [1, k]$, the \kgreedy{} is a $(1 - e^{ -1/ \lambda})/(3 \max(1, \alpha))$-approximation algorithm for Problem~\ref{problem:static}. Its run time is $\bigo{n^{\max(k/\lambda, 2)}}$.
\end{theorem}
This proof is based on the work of Khuller et al. \cite{DBLP:journals/ipl/KhullerMN99}. Note that if the function $f$ is monotone, then the approximation guarantee given in Theorem~\ref{theorem:static} matches well-known results \cite{DBLP:journals/ipl/KhullerMN99}. We remark that non-monotone functions with bounded curvature are not uncommon in practise. For instance, all cut functions of directed graphs are non-monotone, submodular and have curvature $\alpha \leq 2$, as discussed in \cite{DBLP:journals/corr/abs-1811-05351}.

We perform the run time analysis for the \kgreedy{} in dynamic settings, in which weights $\{ W_i \}_i$ change over time.
\begin{theorem}
\label{theorem:dynamic}
Consider Algorithm~\ref{alg:dynamic} optimizing a submodular function with curvature $\alpha > 0$ and knapsack constraints $(c, W )$. Suppose that at some point during the optimization process new weights $W'_i$ are given. Let $\sigma$ be the current solution before the weights update, and let $\sigma_t \subseteq \sigma$ consist of the first $t$ points added to $\sigma$. Furthermore, let $t^*$ be the largest index such that $\sigma_{t^*} \subseteq \mathcal{V} \cap \mathcal{V}'$, with $\mathcal{V}, \mathcal{V}'$ as in Algorithm~\ref{alg:dynamic}, and define 
$\chi = \min \left  \{ \chi(c, W), \chi (c, W'), t^* \right \}.$
Then after additional $\bigo{n(n - \chi)}$ run time the \kgreedy{} finds a $(1 - e^{ -1/ \lambda})/(3 \max(1, \alpha))$-approximate optimal solution in the new constraints, for a fixed parameter $\lambda \in [1, k]$.
\end{theorem}
Intuitively, this proof shows that, given two sets of feasible solutions $\mathcal{V}, \mathcal{V}'$, the \kgreedy{} follows the same paths on both problem instances, for all solutions with size up to $\chi$. Note that Theorem \ref{theorem:static} yields the same theoretical approximation guarantee as Theorem \ref{theorem:dynamic}. Hence, if dynamic updates occur at a slow pace, it is possible to obtain identical results by restarting $\kgreedy{}$ every time a constraint update occurs. However, as we show in Section \ref{sec:experiments_dynamic}, there is significant advantage in using the \dgreedy{} in settings when frequent noisy constraints updates occur.

\section{APPLICATIONS}
\label{sec:experimental_framework}
In this section we present a high-level overview of two possible applications for Problem \ref{problem:static}. We describe experimental frameworks and implementations for these applications in Sections \ref{sec:experiments_static}-\ref{sec:experiments_dynamic}.

\paragraph{Video Summarization.} Determinantal Point Process (DPP), \cite{macchi_1975}, is a probabilistic model, the probability distribution function of which can be characterized as the determinant of a matrix. More formally, consider a sample space $V = [n]$, and let $L$ be a positive semidefinite matrix. We say that $L$ defines a DPP on $V$, if the probability of an event $S \subseteq V$ is given by
\begin{equation*}
\label{eq:DPP}
\mathcal{P}_L(S) = \frac{\det(L_S)}{\det(L+I)},
\end{equation*}
where $L_S = (L_{i,j})_{i,j \in S}$ is the submatrix of $L$ indexed by the elements in $S$, and $I$ is the $n\times n$ identity matrix. For a survey on DPPs and their applications see \cite{DBLP:journals/ftml/KuleszaT12}.

We model this framework with a matrix $L$ that describes similarities between pairs of frames. Intuitively, if $L$ describes the similarity between two frames, then the DPP prefers diversity.\\
In this setting, we search for a set of features $S \subseteq V$ such that $\mathcal{P}(S)$ is maximal, among sets of feasible solutions defined in terms of a knapsack constraint. Since $L$ is positive semidefinite, the function $\log \det L_S$ is submodular \cite{DBLP:journals/ftml/KuleszaT12}. 
\paragraph{Sensor Placement.} The maximum entropy sampling problem consists of choosing the most informative subset of random variables subject to side constraints. In this work, we study the problem of finding the most informative set among given Gaussian time series. 

Given a sample covariance matrix $\Sigma$ of the time series corresponding to measurements, the entropy of a subset of sensors is then given by the formula
\[
f(S) = \frac{1 + \ln (2\pi)}{2} \absl{S} + \frac{1}{2} \ln \det(\Sigma_S)
\]
for any indexing set $S \subseteq \{0, 1\}^n$ on the variation series, where $\det(\Sigma_S)$ returns the determinant of the sub-matrix of $\Sigma$ indexed by $S$. It is well-known that the function $f$ is non-monotone and submodular. Its curvature is bounded as $\alpha \leq 1 - 1/\mu$, where $\mu$ is the largest eigenvalue of $\Sigma$~\cite{DBLP:journals/mor/SviridenkoVW17,DBLP:journals/jmlr/KrauseSG08,DBLP:journals/corr/abs-1811-05351}.\\
We consider the problem of maximizing the entropy $f$ under a \emph{partition matroid constraint}. This additional side constraint requires upper-bounds on the number of sensors that can be chosen in given geographical areas.

\section{STATIC EXPERIMENTS}
\label{sec:experiments_static}
\begin{table*}[t]
\begin{tabular}{c c c c c c c c c}
& & \multicolumn{7}{c}{\textbf{run time for each video clip}}  \\ \cline{3-9}
\textbf{algorithm} & \textbf{parameter} & \multicolumn{1}{|c|}{video (1)} & \multicolumn{1}{c|}{video (2)} & \multicolumn{1}{c|}{video (3)} & \multicolumn{1}{c|}{video (4)} & \multicolumn{1}{c|}{video (5)} & \multicolumn{1}{c|}{video (6)} & \multicolumn{1}{c|}{video (7)} \\ \hline
\multicolumn{1}{|c|}{\fantom{}} & $\varepsilon = 0.001$  & \multicolumn{1}{|c|}{4085317} & \multicolumn{1}{c|}{3529230} & \multicolumn{1}{c|}{3813637} & \multicolumn{1}{c|}{2986719} & \multicolumn{1}{c|}{3368901} & \multicolumn{1}{c|}{3442329} & \multicolumn{1}{c|}{3082814} \\ 
\multicolumn{1}{|c|}{\fantom{}} & $\varepsilon = 0.01$  & \multicolumn{1}{|c|}{406960} & \multicolumn{1}{c|}{351321} & \multicolumn{1}{c|}{382342} & \multicolumn{1}{c|}{299444} & \multicolumn{1}{c|}{336670} & \multicolumn{1}{c|}{344619} & \multicolumn{1}{c|}{307726} \\ 
\multicolumn{1}{|c|}{\fantom{}} & $\varepsilon = 0.1$ & \multicolumn{1}{|c|}{41008} & \multicolumn{1}{c|}{34309} & \multicolumn{1}{c|}{37232} & \multicolumn{1}{c|}{29249} & \multicolumn{1}{c|}{33343} & \multicolumn{1}{c|}{33146} & \multicolumn{1}{c|}{30235} \\ 
\multicolumn{1}{|c|}{\kgreedy{}} & $\lambda = k $ & \multicolumn{1}{|c|}{2895} & \multicolumn{1}{c|}{2895} & \multicolumn{1}{c|}{2895} & \multicolumn{1}{c|}{2709} & \multicolumn{1}{c|}{2895} & \multicolumn{1}{c|}{2709} & \multicolumn{1}{c|}{2522} \\ \hline
\multicolumn{9}{c}{} \\ \cline{3-9}
\textbf{algorithm} & \textbf{parameter} & \multicolumn{1}{|c|}{video (8)} & \multicolumn{1}{c|}{video (9)} & \multicolumn{1}{c|}{video (10)} & \multicolumn{1}{c|}{video (11)} & \multicolumn{1}{c|}{video (12)} & \multicolumn{1}{c|}{video (13)} & \multicolumn{1}{c|}{video (14)} \\ \hline
\multicolumn{1}{|c|}{\fantom{}} & $\varepsilon = 0.001$  & \multicolumn{1}{|c|}{3171467} & \multicolumn{1}{c|}{3781141} & \multicolumn{1}{c|}{3121379} & \multicolumn{1}{c|}{3673776} & \multicolumn{1}{c|}{3603787} & \multicolumn{1}{c|}{3055122} & \multicolumn{1}{c|}{4119203} \\ 
\multicolumn{1}{|c|}{\fantom{}} & $\varepsilon = 0.01$  & \multicolumn{1}{|c|}{320738} & \multicolumn{1}{c|}{379420} & \multicolumn{1}{c|}{313975} & \multicolumn{1}{c|}{368617} & \multicolumn{1}{c|}{360265} & \multicolumn{1}{c|}{ 307399} & \multicolumn{1}{c|}{412247} \\ 
\multicolumn{1}{|c|}{\fantom{}} & $\varepsilon = 0.1$ & \multicolumn{1}{|c|}{30605} & \multicolumn{1}{c|}{36608} & \multicolumn{1}{c|}{30809} & \multicolumn{1}{c|}{36242} & \multicolumn{1}{c|}{35682} & \multicolumn{1}{c|}{30223} & \multicolumn{1}{c|}{40543} \\ 
\multicolumn{1}{|c|}{\kgreedy{}} & $\lambda = k $ & \multicolumn{1}{|c|}{2709} & \multicolumn{1}{c|}{3080} & \multicolumn{1}{c|}{2895} & \multicolumn{1}{c|}{2895} & \multicolumn{1}{c|}{3080} & \multicolumn{1}{c|}{2709} & \multicolumn{1}{c|}{3080} \\ \hline
\multicolumn{9}{c}{} \\ \cline{3-8}
\textbf{algorithm} & \textbf{parameter} & \multicolumn{1}{|c|}{video (15)} & \multicolumn{1}{c|}{video (16)} & \multicolumn{1}{c|}{video (17)} & \multicolumn{1}{c|}{video (18)} & \multicolumn{1}{c|}{video (19)} & \multicolumn{1}{c|}{video (20)} & \\ \cline{1-8}
\multicolumn{1}{|c|}{\fantom{}} & $\varepsilon = 0.001$  & \multicolumn{1}{|c|}{3727241} & \multicolumn{1}{c|}{3321987} & \multicolumn{1}{c|}{3429387} & \multicolumn{1}{c|}{3555884} & \multicolumn{1}{c|}{3375296 } & \multicolumn{1}{c|}{3431653} &  \\ 
\multicolumn{1}{|c|}{\fantom{}} & $\varepsilon = 0.01$  & \multicolumn{1}{|c|}{374322} & \multicolumn{1}{c|}{330994} & \multicolumn{1}{c|}{345333} & \multicolumn{1}{c|}{354694} & \multicolumn{1}{c|}{336718} & \multicolumn{1}{c|}{343419} & \\ 
\multicolumn{1}{|c|}{\fantom{}} & $\varepsilon = 0.1$ & \multicolumn{1}{|c|}{36073} & \multicolumn{1}{c|}{32377} & \multicolumn{1}{c|}{33357} & \multicolumn{1}{c|}{34741} & \multicolumn{1}{c|}{32370} & \multicolumn{1}{c|}{32187} & \\ 
\multicolumn{1}{|c|}{\kgreedy{}} & $\lambda = k $ & \multicolumn{1}{|c|}{2895} & \multicolumn{1}{c|}{2895} & \multicolumn{1}{c|}{2895} & \multicolumn{1}{c|}{2895} & \multicolumn{1}{c|}{2709} & \multicolumn{1}{c|}{2709} & \\ \cline{1-8}
\multicolumn{9}{c}{}  \\
\end{tabular}
\begin{tabular}{c c c c c c c c c}\label{tbl:static_solutionquality}
& & \multicolumn{7}{c}{\textbf{solution quality for each video clip}}  \\ \cline{3-9}
\textbf{algorithm} & \textbf{parameter} & \multicolumn{1}{|c|}{video (1)} & \multicolumn{1}{c|}{video (2)} & \multicolumn{1}{c|}{video (3)} & \multicolumn{1}{c|}{video (4)} & \multicolumn{1}{c|}{video (5)} & \multicolumn{1}{c|}{video (6)} & \multicolumn{1}{c|}{video (7)} \\ \hline
\multicolumn{1}{|c|}{\fantom{}} & $\varepsilon = 0.001$  & \multicolumn{1}{|c|}{19.3818} & \multicolumn{1}{c|}{15.6143} & \multicolumn{1}{c|}{15.4285} & \multicolumn{1}{c|}{10.6228} & \multicolumn{1}{c|}{13.2393} & \multicolumn{1}{c|}{14.0438} & \multicolumn{1}{c|}{9.4999} \\ 
\multicolumn{1}{|c|}{\fantom{}} & $\varepsilon = 0.01$  & \multicolumn{1}{|c|}{19.3818} & \multicolumn{1}{c|}{15.6143} & \multicolumn{1}{c|}{15.4285} & \multicolumn{1}{c|}{10.6228} & \multicolumn{1}{c|}{13.2393} & \multicolumn{1}{c|}{12.9851} & \multicolumn{1}{c|}{9.4999} \\ 
\multicolumn{1}{|c|}{\fantom{}} & $\varepsilon = 0.1$ & \multicolumn{1}{|c|}{16.9083} & \multicolumn{1}{c|}{13.9868} & \multicolumn{1}{c|}{13.9942} & \multicolumn{1}{c|}{9.1811} & \multicolumn{1}{c|}{13.2393} & \multicolumn{1}{c|}{12.9851} & \multicolumn{1}{c|}{9.4999} \\ 
\multicolumn{1}{|c|}{\kgreedy{}} & $\lambda = k $ & \multicolumn{1}{|c|}{23.9323} & \multicolumn{1}{c|}{21.8122} & \multicolumn{1}{c|}{22.5406} & \multicolumn{1}{c|}{15.0203} & \multicolumn{1}{c|}{19.4932} & \multicolumn{1}{c|}{18.6267} & \multicolumn{1}{c|}{15.5678} \\ \hline
\multicolumn{9}{c}{} \\ \cline{3-9}
\textbf{algorithm} & \textbf{parameter} & \multicolumn{1}{|c|}{video (8)} & \multicolumn{1}{c|}{video (9)} & \multicolumn{1}{c|}{video (10)} & \multicolumn{1}{c|}{video (11)} & \multicolumn{1}{c|}{video (12)} & \multicolumn{1}{c|}{video (13)} & \multicolumn{1}{c|}{video (14)} \\ \hline
\multicolumn{1}{|c|}{\fantom{}} & $\varepsilon = 0.001$  & \multicolumn{1}{|c|}{11.0898} & \multicolumn{1}{c|}{16.2864} & \multicolumn{1}{c|}{10.7798} & \multicolumn{1}{c|}{15.9894} & \multicolumn{1}{c|}{15.5939} & \multicolumn{1}{c|}{12.5897} & \multicolumn{1}{c|}{18.7495} \\ 
\multicolumn{1}{|c|}{\fantom{}} & $\varepsilon = 0.01$  & \multicolumn{1}{|c|}{11.0898} & \multicolumn{1}{c|}{16.2864} & \multicolumn{1}{c|}{10.7798} & \multicolumn{1}{c|}{15.9894} & \multicolumn{1}{c|}{15.5939} & \multicolumn{1}{c|}{12.5897} & \multicolumn{1}{c|}{18.7495} \\ 
\multicolumn{1}{|c|}{\fantom{}} & $\varepsilon = 0.1$ & \multicolumn{1}{|c|}{9.5612} & \multicolumn{1}{c|}{16.2864} & \multicolumn{1}{c|}{10.7798} & \multicolumn{1}{c|}{14.4139} & \multicolumn{1}{c|}{14.1122} & \multicolumn{1}{c|}{9.5909} & \multicolumn{1}{c|}{17.3095} \\ 
\multicolumn{1}{|c|}{\kgreedy{}} & $\lambda = k $ & \multicolumn{1}{|c|}{18.5727} & \multicolumn{1}{c|}{23.9619} & \multicolumn{1}{c|}{17.6612} & \multicolumn{1}{c|}{23.2229} & \multicolumn{1}{c|}{20.8876} & \multicolumn{1}{c|}{18.1164} & \multicolumn{1}{c|}{25.1342} \\ \hline
\multicolumn{9}{c}{} \\ \cline{3-8}
\textbf{algorithm} & \textbf{parameter} & \multicolumn{1}{|c|}{video (15)} & \multicolumn{1}{c|}{video (16)} & \multicolumn{1}{c|}{video (17)} & \multicolumn{1}{c|}{video (18)} & \multicolumn{1}{c|}{video (19)} & \multicolumn{1}{c|}{video (20)} & \\ \cline{1-8}
\multicolumn{1}{|c|}{\fantom{}} & $\varepsilon = 0.001$  & \multicolumn{1}{|c|}{16.3391} & \multicolumn{1}{c|}{11.8452} & \multicolumn{1}{c|}{13.6084} & \multicolumn{1}{c|}{16.9964} & \multicolumn{1}{c|}{13.0314} & \multicolumn{1}{c|}{13.0558} &  \\ 
\multicolumn{1}{|c|}{\fantom{}} & $\varepsilon = 0.01$  & \multicolumn{1}{|c|}{16.3391} & \multicolumn{1}{c|}{11.8452} & \multicolumn{1}{c|}{13.6084} & \multicolumn{1}{c|}{16.9964} & \multicolumn{1}{c|}{13.0314} & \multicolumn{1}{c|}{13.0558} &  \\ 
\multicolumn{1}{|c|}{\fantom{}} & $\varepsilon = 0.1$ & \multicolumn{1}{|c|}{14.7544} & \multicolumn{1}{c|}{11.8452} & \multicolumn{1}{c|}{12.0878} & \multicolumn{1}{c|}{14.0999} & \multicolumn{1}{c|}{11.5619} & \multicolumn{1}{c|}{11.5385} & \\ 
\multicolumn{1}{|c|}{\kgreedy{}} & $\lambda = k $ & \multicolumn{1}{|c|}{23.8740} & \multicolumn{1}{c|}{19.4916} & \multicolumn{1}{c|}{19.9461} & \multicolumn{1}{c|}{22.2884} & \multicolumn{1}{c|}{18.6588} & \multicolumn{1}{c|}{19.0673} & \\ \cline{1-8}
\end{tabular}
\caption{We consider $20$ movies from the Frames Labeled In Cinema dataset \cite{FLICWeb}. For each movie, we select a representative set of frames, by maximizing a submodular function under additional knapsack constraints, as described in Section \ref{sec:experiments_static}. We run the \kgreedy{} and the \fantomalg{} algorithm \cite{MirzasoleimanBK16} with various parameter choices, until no remaining point in the search space yields improvement on the fitness value, without violating side constraints. We observe that in all cases the \kgreedy{} yields better run time and solution quality than the \fantomalg{}.}\label{tbl:static_runtime}
\end{table*}
The aim of these experiments is to show that the \kgreedy{} yields better performance in comparison with \fantomalg{} \cite{MirzasoleimanBK16}, which is a popular algorithm for non-monotone submodular objectives under complex sets of constraints. We consider video summarization tasks as in Section \ref{sec:experimental_framework}.

Let $L$ be the matrix describing similarities between pairs of frames, as in Section \ref{sec:experimental_framework}. Following \cite{DBLP:conf/nips/GongCGS14}, we parametrize $L$ as follows. Given a set of frames, let $\bm{f}_i$ be the feature vector of the $i$-th frame. This vector encodes the contextual information about frame $i$ and its representativeness of other items. Then the matrix $L$ can be paramterized as $L_{i,j} = \bm{z}_i^T W^T W \bm{z}_j$,
with $\bm{z}_i = \tanh(U \bm{f}_i)$ being a hidden representation of $\bm{f}_i$, and $U, W$ parameters. We use a single-layer neural network to train the parameters $U, W$. We consider $20$ movies from the Frames Labeled In Cinema dataset \cite{FLICWeb}. Each movie has $200$ frames and $7$ generated ground summaries consisting of $15$ frames each.

We select a representative set of frames, by maximizing the function $\log \det L$ under additional quality feature constraints, viewed as multiple knapsacks. Hence, this task consists of maximizing a non-monotone submodular function under multiple knapsack constraints.\\ 
We run the \kgreedy{} and \fantomalg{} algorithms on each instance, until no remaining point in the search space yields improvement on the fitness value, without violating side constraints. We then compare the resulting run time and empirical approximation ratio. Since \fantomalg{} depends on a parameter $\varepsilon$ \cite{MirzasoleimanBK16}, we perform three sets of experiments for $\varepsilon = 0.1$, $\varepsilon =0.01$, and $\varepsilon =0.001$. The parameter $\lambda$ for the \kgreedy{} is always set to $\lambda = k$. We have no indications that a lower $\lambda$ yields improved solution quality on this set of instances.

Results for the run time and approximation guarantee are displayed in Table \ref{tbl:static_runtime}. We clearly see that the $\kgreedy{}$ outperforms $\fantom{}$ in terms of solution quality. Furthermore, the run time of $\fantom{}$ is orders of magnitude worse than that of our $\kgreedy{}$. This is probably due to the fact that the $\fantom{}$ requires a very low density threshold to get to a good solution on these instances. The code for this set of experiments is available upon request.

\section{DYNAMIC EXPERIMENTS}
\label{sec:experiments_dynamic}
The aim of these experiments is to show that, when constraints quickly change dynamically, the robust \dgreedy{} significantly outperforms the $\kgreedy{}$ with a restart policy, that re-sets the optimization process each time new weights are given. To this end, we simulate a setting where updates change dynamically, by introducing controlled posterior noise on the weights. At each update, we run the $\kgreedy{}$ from scratch, and let the \dgreedy{} continue without a restart policy. We consider two set of dynamic experiments.

\paragraph{The Maximum Entropy Sampling Problem}
\begin{table*}[t]
\caption{For the respective set of experiments, Maximum Entropy Sampling (left) and Video Summarization (right), the mean and standard deviation of the solution quality over time within one run of the respective algorithms for different update frequencies [$\tau$] and different dynamic update standard deviations [$\sigma$]. With $X^{(+)}$ we denote that the algorithm labelled $X$ significantly outperformed the other one.}
\begin{tabular}{c c c c c c c}\cline{3-6}\label{tbl:greedy_matroid}\label{tbl:greedy_knapsack}
                           &                            & \multicolumn{2}{|c|}{\kgreedy{} (1)} & \multicolumn{2}{|c|}{\dgreedy{} (2)} &  \\ \hline
\multicolumn{1}{|c}{$\tau$} & \multicolumn{1}{c|}{$\sigma$}              & mean    & \multicolumn{1}{c|}{sd}    & mean    & \multicolumn{1}{c|}{sd}     &       \multicolumn{1}{c|}{stat}                   \\ \hline
\multicolumn{1}{|c}{10K}                    & \multicolumn{1}{c|}{0.075} & 107.4085 & \multicolumn{1}{r|}{1.35}  & $\mathbf{130.5897}$   & \multicolumn{1}{c|}{12.28}                   & \multicolumn{1}{c|}{$2^{(+)}$}           \\
\multicolumn{1}{|c}{20K}                        & \multicolumn{1}{c|}{0.075} & 195.0352 & \multicolumn{1}{r|}{2.99}  & $\mathbf{213.3886}$   & \multicolumn{1}{c|}{11.15}                   & \multicolumn{1}{c|}{$2^{(+)}$}                \\
\multicolumn{1}{|c}{30K}                        & \multicolumn{1}{c|}{0.075} & 263.8143 & \multicolumn{1}{r|}{4.80}  & $\mathbf{279.7259}$   & \multicolumn{1}{c|}{11.38}                   & \multicolumn{1}{c|}{$2^{(+)}$}                \\
\multicolumn{1}{|c}{40K}                        & \multicolumn{1}{c|}{0.075} & 319.6090 & \multicolumn{1}{r|}{6.78}  & $\mathbf{329.8283}$   & \multicolumn{1}{c|}{10.80}                   & \multicolumn{1}{c|}{$2^{(+)}$}                \\
\multicolumn{1}{|c}{50K}                        & \multicolumn{1}{c|}{0.075} & 351.6649 & \multicolumn{1}{r|}{8.47}  & $\mathbf{353.6060}$   & \multicolumn{1}{c|}{9.19}                    & \multicolumn{1}{c|}{$2^{(+)}$}                \\ \hline
\multicolumn{1}{|c}{10K}                        & \multicolumn{1}{c|}{0.05}  & 107.6904 & \multicolumn{1}{r|}{0.73}  & $\mathbf{132.7425}$   & \multicolumn{1}{c|}{8.79}                    & \multicolumn{1}{c|}{$2^{(+)}$}                \\
\multicolumn{1}{|c}{20K}                        & \multicolumn{1}{c|}{0.05}  & 195.5784 & \multicolumn{1}{r|}{1.82}  & $\mathbf{216.0143}$   & \multicolumn{1}{c|}{8.00}                    & \multicolumn{1}{c|}{$2^{(+)}$}                \\
\multicolumn{1}{|c}{30K}                        & \multicolumn{1}{c|}{0.05}  & 264.6016 & \multicolumn{1}{r|}{3.13}  & $\mathbf{282.2722}$   & \multicolumn{1}{c|}{7.91}                    & \multicolumn{1}{c|}{$2^{(+)}$}                \\
\multicolumn{1}{|c}{40K}                        & \multicolumn{1}{c|}{0.05}  & 320.6179 & \multicolumn{1}{r|}{4.66}  & $\mathbf{332.0804}$   & \multicolumn{1}{c|}{7.57}                    & \multicolumn{1}{c|}{$2^{(+)}$}                \\
\multicolumn{1}{|c}{50K}                        & \multicolumn{1}{c|}{0.05}  & 352.8980 & \multicolumn{1}{r|}{6.10}  & $\mathbf{355.0595}$   & \multicolumn{1}{c|}{6.64}                    & \multicolumn{1}{c|}{$2^{(+)}$}                \\ \hline
\multicolumn{1}{|c}{10K}                        & \multicolumn{1}{c|}{0.10}  & 107.2164 & \multicolumn{1}{r|}{1.56}  & $\mathbf{124.4393}$   & \multicolumn{1}{c|}{14.81}                   & \multicolumn{1}{c|}{$2^{(+)}$}                \\
\multicolumn{1}{|c}{20K}                        & \multicolumn{1}{c|}{0.10}  & 194.4608 & \multicolumn{1}{r|}{3.43}  & $\mathbf{208.3658}$   & \multicolumn{1}{c|}{13.35}                   & \multicolumn{1}{c|}{$2^{(+)}$}                \\
\multicolumn{1}{|c}{30K}                        & \multicolumn{1}{c|}{0.10}  & 262.7148 & \multicolumn{1}{r|}{5.57}  & $\mathbf{274.8722}$   & \multicolumn{1}{c|}{13.64}                   & \multicolumn{1}{c|}{$2^{(+)}$}                \\
\multicolumn{1}{|c}{40K}                        & \multicolumn{1}{c|}{0.10}  & 317.9336 & \multicolumn{1}{r|}{7.96}  & $\mathbf{325.7493}$   & \multicolumn{1}{c|}{12.93}                   & \multicolumn{1}{c|}{$2^{(+)}$}                \\
\multicolumn{1}{|c}{50K}                        & \multicolumn{1}{c|}{0.10}  & 349.5698 & \multicolumn{1}{r|}{10.04} & $\mathbf{351.1627}$   & \multicolumn{1}{c|}{10.91}                   & \multicolumn{1}{c|}{$2^{(+)}$}                \\ \hline
\end{tabular}
\hspace{.2cm}
\begin{tabular}{c c c c c c c}\cline{3-6}
                           &                            & \multicolumn{2}{|c|}{\kgreedy{} (1)} & \multicolumn{2}{c|}{\dgreedy{} (2)} &  \\ \hline 
\multicolumn{1}{|c}{$\tau$} & $\sigma$                   & \multicolumn{1}{|c}{mean}    & \multicolumn{1}{c|}{sd}    & mean    & \multicolumn{1}{c|}{sd}     &   \multicolumn{1}{c|}{stat}                       \\ \hline
\multicolumn{1}{|c}{10K}    & \multicolumn{1}{c|}{0.075} & \multicolumn{1}{r}{36.550}  & \multicolumn{1}{r|}{2.1E-14} & $\mathbf{264.256} $  & \multicolumn{1}{c|}{69.21}                   & \multicolumn{1}{c|}{$2^{(+)}$}                \\
\multicolumn{1}{|c}{20K}    & \multicolumn{1}{c|}{0.075} & \multicolumn{1}{r}{69.760}  & \multicolumn{1}{r|}{7.2E-14} & $\mathbf{269.639} $  & \multicolumn{1}{c|}{57.77}                   & \multicolumn{1}{c|}{$2^{(+)}$}                \\
\multicolumn{1}{|c}{30K}    & \multicolumn{1}{c|}{0.075} & 102.969                     & \multicolumn{1}{r|}{4.3E-14} & $\mathbf{271.531}$   & \multicolumn{1}{c|}{53.20}                   & \multicolumn{1}{c|}{$2^{(+)}$}                \\
\multicolumn{1}{|c}{40K}    & \multicolumn{1}{c|}{0.075} & 136.956                     & \multicolumn{1}{r|}{8.6E-14} & $\mathbf{272.210}$   & \multicolumn{1}{c|}{51.18}                   & \multicolumn{1}{c|}{$2^{(+)}$}                \\
\multicolumn{1}{|c}{50K}    & \multicolumn{1}{c|}{0.075} & 174.657                     & \multicolumn{1}{r|}{1.30}     & $\mathbf{272.968} $  & \multicolumn{1}{c|}{49.38}                   & \multicolumn{1}{c|}{$2^{(+)}$}                \\ \hline
\multicolumn{1}{|c}{10K}    & \multicolumn{1}{c|}{0.05}  & \multicolumn{1}{r}{36.55}  & \multicolumn{1}{r|}{2.1E-14} & $\mathbf{197.251}$   & \multicolumn{1}{c|}{63.49}                   & \multicolumn{1}{c|}{$2^{(+)}$}                \\
\multicolumn{1}{|c}{20K}    & \multicolumn{1}{c|}{0.05}  & \multicolumn{1}{r}{69.760}  & \multicolumn{1}{r|}{7.9E-14} & $\mathbf{257.641}$   & \multicolumn{1}{c|}{52.15}                   & \multicolumn{1}{c|}{$2^{(+)}$}                \\
\multicolumn{1}{|c}{30K}    & \multicolumn{1}{c|}{0.05}  & 102.969                     & \multicolumn{1}{r|}{4.3E-14} & $\mathbf{259.517}$   & \multicolumn{1}{c|}{47.55}                   & \multicolumn{1}{c|}{$2^{(+)}$}                \\
\multicolumn{1}{|c}{40K}    & \multicolumn{1}{c|}{0.05}  & 136.956                     & \multicolumn{1}{r|}{8.6E-14} & $\mathbf{260.197}$   & \multicolumn{1}{c|}{45.46}                   & \multicolumn{1}{c|}{$2^{(+)}$}                \\
\multicolumn{1}{|c}{50K}    & \multicolumn{1}{c|}{0.05}  & 174.840                     & \multicolumn{1}{r|}{5.7E-14} & $\mathbf{260.955}$   & \multicolumn{1}{c|}{43.64}                   & \multicolumn{1}{c|}{$2^{(+)}$}                \\ \hline
\multicolumn{1}{|c}{10K}    & \multicolumn{1}{c|}{0.10}  & \multicolumn{1}{r}{36.550}  & \multicolumn{1}{r|}{2.1E-14} & $\mathbf{269.784}$   & \multicolumn{1}{c|}{76.62}                   & \multicolumn{1}{c|}{$2^{(+)}$}               \\
\multicolumn{1}{|c}{20K}    & \multicolumn{1}{c|}{0.10}  & \multicolumn{1}{r}{69.760}  & \multicolumn{1}{r|}{7.2E-14} & $\mathbf{275.981}$   & \multicolumn{1}{c|}{65.69}                   & \multicolumn{1}{c|}{$2^{(+)}$}                \\
\multicolumn{1}{|c}{30K}    & \multicolumn{1}{c|}{0.10}  & \multicolumn{1}{r}{102.969} & \multicolumn{1}{r|}{4.3E-14} & $\mathbf{277.891}$   & \multicolumn{1}{c|}{61.48}                   & \multicolumn{1}{c|}{$2^{(+)}$}                \\
\multicolumn{1}{|c}{40K}    & \multicolumn{1}{c|}{0.10}  & \multicolumn{1}{r}{136.956} & \multicolumn{1}{r|}{8.6E-14} & $\mathbf{278.571}$   & \multicolumn{1}{c|}{59.67}                   & \multicolumn{1}{c|}{$2^{(+)}$}                \\
\multicolumn{1}{|c}{50K}    & \multicolumn{1}{c|}{0.10}  & 174.448                     & \multicolumn{1}{r|}{2.77}     & 2$\mathbf{79.329}$   & \multicolumn{1}{c|}{58.05}                   & \multicolumn{1}{c|}{$2^{(+)}$}                \\ \hline
\end{tabular}
\end{table*}
\begin{figure*}[t]
    \includegraphics[width=0.49\textwidth]{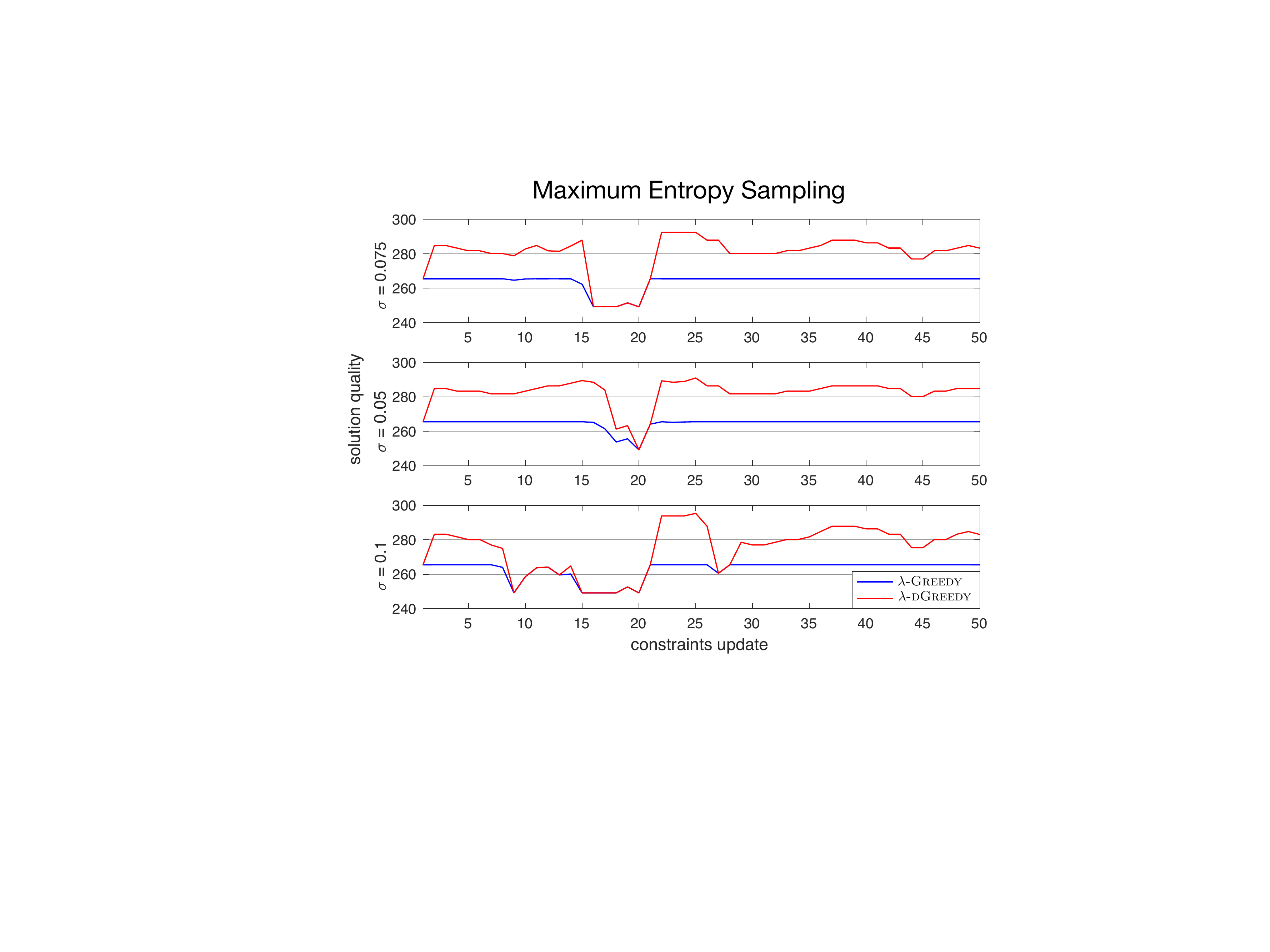}
    \includegraphics[height=182px, width=0.49\textwidth]{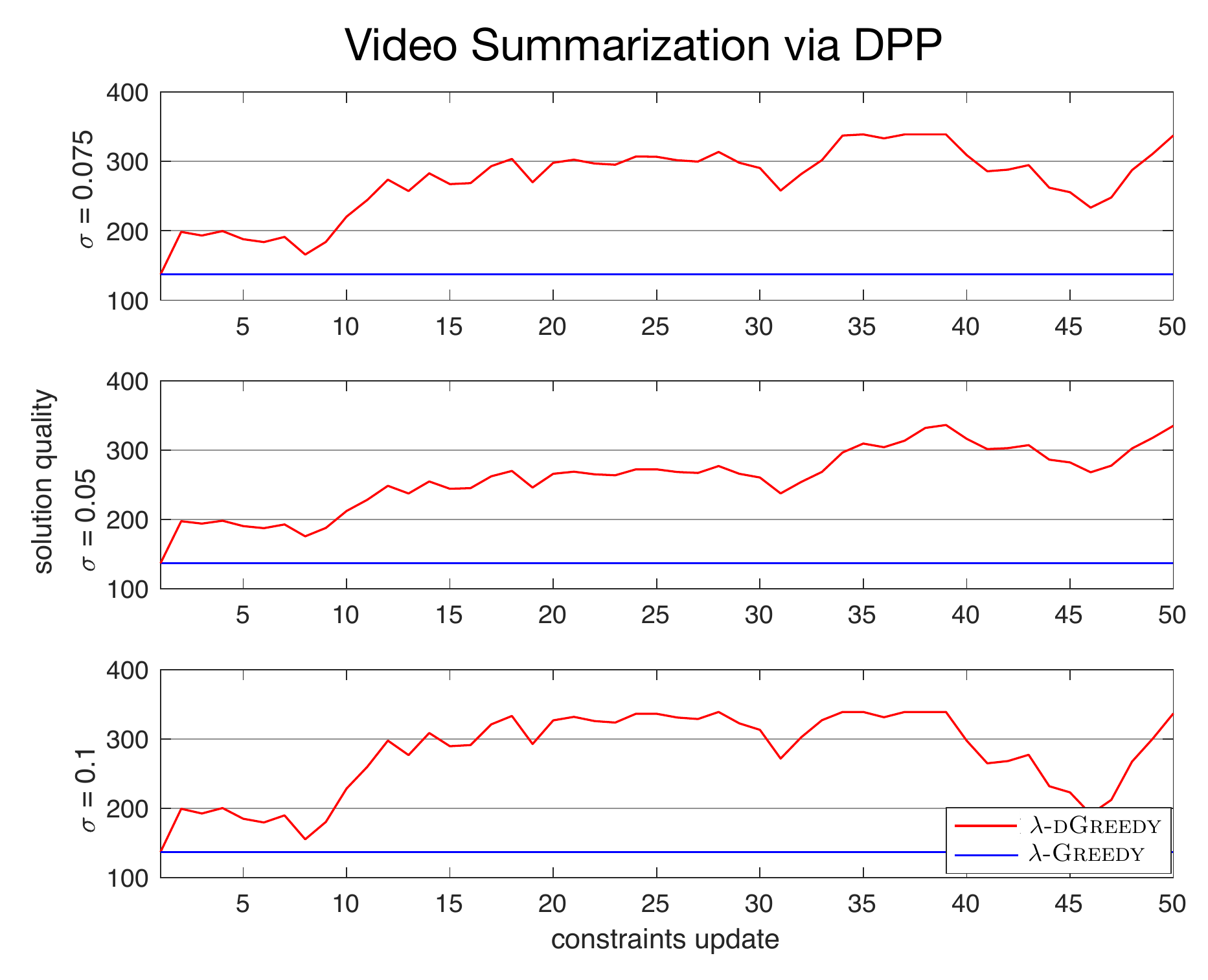}
\caption{The solution quality achieved by the \kgreedy{} after each constraints' update, using a restart strategy, and the \dgreedy{} operator. Each plot shows results for a fixed standard deviation choice [$\sigma$] and fixed update frequency $\tau = 30K$ (for the Maximum Entropy Sampling) and $\tau = 50K$ (for Video Summarization).}\label{fig:plot_matroid}\label{fig:plot_knapsack}
\end{figure*}
We consider the problem of maximizing the entropy $f$ under a partition matroid constraint. This additional side constraint requires an upper bound on the number of sensors that can be chosen in given geographical areas. Specifically, we partition the total number of time series in seven sets, based on the continent in which the corresponding stations are located. Under this partition set, we then have seven independent cardinality constraints, one for each continent.\\
We use the Berkeley Earth Surface Temperature Study, which combines $1.6$ billion temperature reports from $16$ preexisting data archives. This archive contains over $39000$ unique stations from around the world. More information on the Berkeley Earth project can be found in~\cite{BerkeleyWeb}. Here, we consider unique time series defined as the average monthly temperature for each station. Taking into account all data between years $2015$-$2017$, we obtain $2736$ time series from the corresponding stations. Our experimental framework follows along the lines of~\cite{DBLP:journals/corr/abs-1811-05351}.

In our dynamic setting, for each continent, a given parameter $d_i$ is defined as a percentage value of the overall number of stations available on that continent, for all $i \in [7]$. We let parameters $d_1, \dots, d_7$ vary over time, as to simulate a setting where they are updated dynamically. This situation could occur when operational costs slightly vary overtime. We initially set all parameters to use $50\%$ of the available resources, and we introduce a variation of these parameters at regular intervals, according to $\mathcal{N}(0, \sigma^2)$, a Gaussian distribution of mean $0$ and variance $\sigma^2$, for all $i \in [7]$.\\ 
We consider various choices for the standard deviation $\sigma$, and various choices for the time span between one dynamic update and the next one (the parameter $\tau$). For each choice of $\sigma$ and $\tau$, we consider a total of $50$ sequences of changes. We perform statistical validation using the Kruskal-Wallis test with $95$\% confidence. In order to compare the results, we use the Bonferroni post-hoc statistical procedure. This method is used for multiple comparisons of a control algorithm against two or more other algorithms. We refer the reader to \cite{Corder09} for more detailed descriptions of these statistical tests.

We compare the results in terms of the solution quality achieved at each dynamic update by the \kgreedy{} and the \dgreedy{}. We summarize our results in the Table~\ref{tbl:greedy_matroid} (left) as follows. The columns correspond to the results for \kgreedy{} and the \dgreedy{} respectively, along with the mean value, standard deviation, and statistical comparison. The symbol $X^{(+)}$ is equivalent to the statement that the algorithm labelled as $X$ significantly outperformed the other one.\\
Table~\ref{tbl:greedy_matroid} (left) shows that the \dgreedy{} has a better performance than the \kgreedy{} algorithm with restarts, when dynamic changes occur, especially for the highest frequencies $\tau = 10K, 20K$. This shows that the \dgreedy{} is suitable in settings when frequent dynamic changes occur. The \kgreedy{} yields improved performance with lower frequencies, but it under-perform the \dgreedy{} on our dataset.\\
Figure~\ref{fig:plot_matroid} (left) shows the solution quality values achieved by the \kgreedy{} and the \dgreedy{}, for different choices of the standard deviation $\sigma= 0.075, 0.05, 0.1$. Again, we observe that the \dgreedy{} finds solutions that have better quality than the \kgreedy{} with restarts. Even though the \dgreedy{} in some cases aligns with the \kgreedy{} with restarts, the performance of the \dgreedy{} is clearly better than that of the simple \kgreedy{} with restarts. The code for this set of experiments is available upon request.

\paragraph{Determinantal Point Processes}
We conclude with a dynamic set of experiments on a video summarization task as in Section \ref{sec:experimental_framework}. We define the corresponding matrix $L$ using the quality-diversity decomposition, as proposed in \cite{DBLP:conf/nips/KuleszaT10}. Specifically, we define the coefficients $L_{i,j}$ of this matrix as
$L_{i,j} = q(i) k(i,j) q(j)$, with $q(i)$ representing the quality of the $i$-th frame and $k(i,j)$ being the diversity between the $i$-th and $j$-th frame.\\
For the diversity measure $k$, we consider commonly used features $f \in F$ for pictures, and we use these features to define corresponding feature vectors $v^f_i$ for each frame $i$. Then the diversity measure is defined as
\[
    k(i,j) = \exp\left( -\sum_{f \in F} \frac{||v^f_i - v^f_j||_2^2}{\sigma_f}\right),
\]
with $\sigma_f$ a parameter for this feature\footnote{In our setting we combine the parameters $\sigma_{\text{\textsc{Color2}}} = \sigma_{\text{\textsc{Color3}}}$ and $\sigma_{\text{\textsc{SIFT256}}} = \sigma_{\text{SIFT512}}$.}. To learn these parameters we use the Markov Chain Monte Carlo (MCMC) method (see \cite{DBLP:conf/icml/AffandiFAT14}).

We use movie clips from the Frames Labeled In Cinema dataset \cite{FLICWeb}. We use 16 movies with 150-550 frames each to learn the parameters and one test movie with approximately 400 frames for our experiments. For each movie, we generate 5-10 samples (depending on the total amount of frames) of sets with 10-20 frames as training data. We then use MCMC on the training data to learn the parameters for each movie. When testing the \kgreedy{} and the \dgreedy{}, we use the sample median of the trained parameters.

In this set of experiments, we consider a constraint by which the set of selected frames must not exceed a memory threshold. We define a cost function $c(S)$ as the sum of the size of each frame in $S$. As each frame comes with its own size in memory, choosing the best frames under certain memory budget is equivalent to maximizing a submodular function under a linear knapsack constraint.\\
The weight $W$ is given range $[0\%, 100\%]$, with respect to the total weight $c(V)$, and it is updated dynamically throughout the optimization process, according to a Gaussian distribution $\mathcal{N}(0, \sigma^2)$, for a given variance $\sigma^2$. This settings simulates a situation by which the overall available memory exhibits small frequent variation. 

We select various parameter choices for the standard deviation $\sigma$, and the frequency $\tau$ with which a dynamic update occurs. We investigate the settings $\sigma$ = $0.075$, $0.05$, $0.1$, and $\tau$ = $10K$, $20K$, $30K$, $40K$, $50K$. Each combination of $\sigma$ and $\tau$ carries out $50$ dynamic changes. Again, we validate our results using the Kruskal-Wallis test with $95$\% confidence. To compare the obtained results, we apply the Bonferroni post-hoc statistical test \cite{Corder09}.

The results are presented in the Table~\ref{tbl:greedy_knapsack} (right). We observe that the \dgreedy{} yields better performance than the \kgreedy{} with restarts when dynamic changes occur. Similar findings are obtained when comparing a different standard deviation choice $\sigma$ = $0.075$, $0.05$, $0.1$. Specifically, for the highest frequency $\tau = 10K$, the \dgreedy{} achieves better results by approximately one order of magnitude.\\
Figure~\ref{fig:plot_knapsack} (right) shows the solution quality values obtained by the \dgreedy{} and the \kgreedy{}, as the frequency is set to $\tau = 50K$. It can be observed that, for $\sigma$ = $0.075$, $0.05$, $0.1$, the \dgreedy{} significantly outperforms the \kgreedy{} with restarts, for almost all $50$ updates. The code for this set of experiments is available upon request.

\section{CONCLUSION}
\label{sec:conclusions}

Many real-world optimization problems can be approached as submodular maximization with multiple knapsack constraints (see Problem \ref{problem:static}). Previous studies for this problem show that it is possible to approach this problem with a variety of heuristics. These heuristics often involve a local search, and require continuous relaxations of the discrete problem, and they are impractical. We propose a simple discrete greedy algorithm (see Algorithm \ref{alg:static}) to approach this problem, that has polynomial run time and yields strong approximation guarantees for functions with bounded curvature (see Definition \ref{def:generalized_curvature} and Theorem~\ref{theorem:static}).

Furthermore, we study the problem of maximizing a submodular function, when knapsack constraints involve dynamic components. We study a setting by which the weights $W_i$ of a given set of knapsack constraints change overtime. To this end, we introduce a robust variation of our $\kgreedy{}$ algorithm that allows for handling dynamic constraints online (see Algorithm \ref{alg:dynamic}). We prove that this operator allows to maintain strong approximation guarantees for functions with bounded curvature, when constraints change dynamically (see Theorem \ref{theorem:dynamic}).

We show that, in static settings, Algorithm \ref{alg:static} competes with \fantomalg{}, which is a popular algorithm for handling these constraints (see Table \ref{tbl:static_runtime}). Furthermore, we show that the \dgreedy{} is useful in dynamic settings. To this end, we compare the \dgreedy{} with the \kgreedy{} combined with a restart policy, by which the optimization process starts from scratch at each dynamic update. We observe that the \dgreedy{} yields significant improvement over a restart in dynamic settings with limited computational time budget (see Figure \ref{fig:plot_matroid} and Table \ref{tbl:greedy_matroid}).

\section*{Acknowledgement}
This work has been supported by the Australian Research Council through grant DP160102401 and by the South Australian Government through the Research Consortium "Unlocking Complex Resources through Lean Processing".

\newpage
\bibliography{ecai}

\newpage
\onecolumn
\section*{Appendix: Missing Proofs}
\subsection*{Proof of Theorem \ref{theorem:static}}

We first observe that, for a given $\lambda$, the \kgreedy{} has two phases. During the first phase, \kgreedy{} adds points to the current solution iteratively. During the second phase, \kgreedy{} finds the optimum among single-element sets and sets containing only elements whose cost is lower-bounded with $c_i(e) > \lambda W_i /\lambda$ for all $i \in [n]$. The greedy procedure requires at most $\bigo{n^2}$ steps, while the second procedure requires at most $\bigo{n^{k/ \lambda}}$ steps, since all feasible solutions such that $c_i(e) > \lambda W_i /k$ consist of at most $k/\lambda$ points. Hence, the resulting run time is $\bigo{n^{\max (k/\lambda, 2)}}$ run time.

We now prove the \kgreedy{} yields the desired approximation guarantee. To this end, we define
\begin{align*}
    \mathcal{V}= V_1 = \{e \in V \colon c_j(e) \leq \lambda W_j/k \ \forall j \in [k]\} \quad \mbox{and} \quad V_2 = V \setminus V_1.
\end{align*}
Without loss of generality we assume that $W_1 = \dots = W_k$. We denote with $W$ the weight of each knapsack. Let $\sigma_t$ be a solution of the greedy phase at time step $t$. Let $r \geq 0$ be the smallest index, such that
\begin{enumerate}
    \item $c_j(\sigma_t) \leq W$ for all $j \in [k]$, for all $t \in [r-1]$;
    \item there exists $j \in [k]$ such that $c_j(\sigma_{r}) > W$.
\end{enumerate}
In other words, $r$ is the first point it time such that the new greedy solution does not fulfill all knapsacks at the same time. We first prove that either the solution $\sigma_{r-1}$ or the point $v^* = \argmax_{e \in V} f(e)$ yields a good approximation guarantee of $\opt \cap V_1$.

To simplify the notation, we define $f_t = f(\sigma_t) - f(\sigma_{t - 1})$ and $v_t = \sigma_t \setminus \sigma_{t - 1}$, for all $t \in [r]$. We have that it holds
\begin{align}
f_{\sigma_{t - 1}}(\opt \cap V_1) & \leq \sum_{e \in \opt \setminus \sigma_{t - 1}} f_{\sigma_{t - 1}}(e)\label{eq:thm1}\\   
& = \sum_{e \in (\opt \cap V_1) \setminus \sigma_{t - 1}} \max_j c_j(e) \frac{f_{\sigma_{t - 1}}(e)}{\max_j c_j(e)}\label{eq:thm2}\\
& \leq \frac{f_t}{\max_j c(v_t)} \sum_{e \in \opt \setminus \sigma_{t - 1}} \max_j c_j(e)\label{eq:thm3}\\
& \leq \frac{\lambda W}{\max_j c(v_i)} f_t \label{eq:thm4}
\end{align}
where \eqref{eq:thm1} follows from the assumption that $f$ is submodular; \eqref{eq:thm3} follows from \eqref{eq:thm2} due to the greedy choice of Algorithm \ref{alg:static}; \eqref{eq:thm4} uses the fact that $c(\opt) \leq W$, together with the fact that $c(e) \leq W\lambda/k$ for all $c \in V_1$, for all $j \in [k]$. Rearranging yields
\begin{equation}
f_t \geq \frac{\max_j c_j(v_i)}{\lambda W} (f(\sigma_{t-1}\cup (\opt \cap V_1)) - f(\sigma_{t-1})).\label{eq:thm41}
\end{equation}
To continue with the proof, we consider the following lemma, which follows along the lines of Lemma 1 in \cite{DBLP:journals/corr/abs-1811-05351}.
\begin{lemma}
\label{lemma:marginal_value}
Following the notation introduced above, define the set $J = \{t \in [n] \colon v_t \in \opt \}$. Then for any subset $T \subseteq V$ it holds 
\[
f(\sigma_{t} \cup T) \geq f(T) + (1 - \alpha) f(\sigma_t) - (1 - \alpha)\sum_{j \in [t]\cap J }f_j.
\]
for all $t \in [r]$.
\end{lemma}
\begin{proof}
From the definition of curvature we have that
\[
f(\sigma_i \cup \opt ) - f(\sigma_{i - 1} \cup \opt) \geq (1 - \alpha )\left (f(\sigma_i) - f(\sigma_{i - 1}) \right ).
\]
for all $i = 1, \dots, t$. It follows that 
\begin{align}
f(\sigma_t \cup \opt) & \geq f(\sigma_{t - 1} \cup \opt) + (1 - \alpha )\left (f(\sigma_t) - f(\sigma_{t - 1}) \right )\label{eq:negative_curve1}\\
& \geq f(\emptyset \cup \opt) + \sum_{j = 1}^{t} (1 - \alpha )\left (f(\sigma_j) - f(\sigma_{j - 1}) \right )\label{eq:negative_curve2}\\
& = f(\opt) + (1 - \alpha )  \left (f(\sigma_t) - f(\emptyset) \right )\label{eq:negative_curve3},
\end{align}
where \eqref{eq:negative_curve2} follows by iteratively applying \eqref{eq:negative_curve1} to the $f(U_{j - 1} \cup \opt)$, and \eqref{eq:negative_curve3} follows by taking the telescopic sum.
\end{proof}
Note that Lemma \ref{lemma:marginal_value} yields $f(\sigma_{t} \cup (\opt \cap V_1)) \geq f(\opt \cap V_1) + (1 - \max(1, \alpha)) f(\sigma_t)$, since $\alpha \in [0, 1]$ if and only if the function $f$ is monotone. Combining this observation with \eqref{eq:thm41} yields
\begin{align*}
f_t &  \geq \frac{\max_j c_j(v_i)}{{\max(1, \alpha)}\lambda W} f(\opt \cap V_1) - \frac{\max_j c_j(v_i)}{\lambda W} \sum_{i \in [t - 1]}f_i,
\end{align*}
where we have simply used the telescopic sum over the $f(\sigma_t)$. Defining $x_t = f_t/f(\opt \cap V_1)$ for all $t \in [r]$ we can write the inequality above as
\begin{equation}
\frac{{\max(1, \alpha)}\lambda W}{\max_j c_j(v_t)}x_t + \max(1, \alpha) \sum_{i \in [t - 1]} x_i \geq 1.\label{eq:thm5}
\end{equation}
We conclude the proof by showing that any array of solutions $(x_1, \dots, x_n)$ with coefficients $x_i \in [0, 1]$ that fulfils the LP as in \eqref{eq:thm5} yields
\begin{equation}
\sum_{t\in[r]} x_t \geq \sum_{t\in[r]} \frac{\max_j c_j(v_t)}{{\max(1, \alpha)}\lambda W}\prod_{i \in [t-1]}\left ( 1 - \frac{\max_j c_j(v_i)}{\lambda W} \right ).\label{eq:thm6}
\end{equation}
In order to prove \eqref{eq:thm6}, since it holds $x_t \in [0, 1]$ for all $t \in [r]$, we can simplify our setting, by studying the system
\begin{equation}
\frac{{\max(1, \alpha)} \lambda W}{\max_j c_j(v_t)} x_t + \max(1, \alpha) \sum_{i \in [t - 1]} x_i = 1.\label{eq:thm7}
\end{equation}
This is due to the fact that the sum of the coefficients of any solution of \eqref{eq:thm7} are upper-bounded by the sum of the coefficients of a solution of \eqref{eq:thm6}. We continue with the following simple lemma.
\begin{lemma}
\label{lemma:recursion}
Let $(x_1, \dots, x_r)$ be a solution of the LP as in \eqref{eq:thm7}. Than it holds
\[
x_t \geq \frac{\max_j c_j(v_t)}{{\max(1, \alpha)} \lambda W}\prod_{i \in [t-1]}\left ( 1 - \frac{\max_j c_j(v_i)}{\lambda W} \right ),
\]
for all $t \in [r]$.
\end{lemma}
\begin{proof}
Define $c_t = \max_j c_j(v_t)/({\max(1, \alpha)}\lambda W)$ for all $t \in [r]$. Then the LP as in \eqref{eq:thm7} can be written as
\[
\frac{x_{t}}{c_{t}} = 1 - \max(1, \alpha) \sum_{i = 1}^{t-1} x_t.
\]
By defining $y_t = x_t/c_t$ for all $t \in [r]$, we have that $y_t - y_{t - 1} = - \max(1, \alpha) x_{t-1} = - \max(1, \alpha) c_{t - 1}y_{t-1}$, and we obtain the following recurrent relation $y_t + (\max(1, \alpha) c_t - 1)y_{t - 1} = 0$, for all $t > 1$. This is a recurrent linear equation with solutions 
\[
y_t = \prod_{j = 1}^{t - 1} (1 - \max(1, \alpha) c_j).
\]
The claim follows, by substituting $x_t$ in the equation above.
\end{proof}
Hence, we have that it holds
\begin{align}
f(\sigma_r) & = f(\opt \cap V_1)\sum_t x_t \label{eq:thm8} \\
& \geq f(\opt \cap V_1)\sum_t \frac{\max_j c_j(v_t)}{{\max(1, \alpha)}kW}\prod_{i \in [t-1]}\left ( 1 -  \frac{\max_j c_j(v_i)}{\lambda W} \right )\label{eq:thm9}\\
& \geq  \frac{f(\opt \cap V_1)}{\max(1, \alpha)} \left (1 -  \prod_{i \in [r]}\left ( 1 - \frac{\max_j c_j(v_i)}{\lambda W} \right ) \right )\label{eq:thm10}\\
& \geq  \frac{f(\opt \cap V_1)}{\max(1, \alpha)} \left (1 -  \exp \left \{ - \sum_{i \in [r]}\frac{\max_j c_j(v_i)}{\lambda W} \right \} \right ),\label{eq:thm11}
\end{align}
where \eqref{eq:thm8} holds by taking the telescopic sum; \eqref{eq:thm9} follows from Lemma~\ref{lemma:recursion}; \eqref{eq:thm10} follows via standard calculations; \eqref{eq:thm11} follows because $1 - x \leq e^{-x}$. Consider an index $\ell$ such that $c_{\ell}(\sigma_r) > W$. We have that it holds 
\begin{align*}
f(\sigma_r) & \geq \frac{1}{\max(1, \alpha)} \left (1 - \exp \left \{ - \max(1, \alpha) \sum_{i \in [r]}\frac{\max_j c_j(v_i)}{\lambda W} \right \} \right )f(\opt \cap V_1) \\
& \geq \frac{1}{\max(1, \alpha)} \left (1 - \exp \left \{  - \max(1, \alpha) \sum_{i \in [r]}\frac{c_\ell(v_i)}{\lambda W} \right \} \right )f(\opt \cap V_1) \\
& \geq \frac{1}{\max(1, \alpha)} \left (1 -  e^{-1/\lambda } \right )f(\opt \cap V_1).
\end{align*}

We conclude by proving that Algorithm \ref{alg:static} yields the the desired approximation guarantee. To this end, let 
\begin{align*}
U=\argmax_{\{ W \subseteq V_2 \colon c_j(U) \leq W_j \ \forall j \in [k] \}}\left\{ f(W) \geq f(\opt \cap V_2)\right\}.
\end{align*}

Hence, following the notation of Algorithm \ref{alg:static}, and denoting with $v*$ the point with maximum $f$-value among the singletons, it follows that
\begin{align*}
\argmax \{f(\sigma_{r - 1}), f(U),f(v^*)\} & \geq \frac{1}{3} (f(\sigma_{r - 1}) + f(U) +f(v^*))\\
& \geq \frac{1}{3} (f(\sigma_{r - 1}) + f(v_r) + f(U))\\
& \geq \frac{1}{3} f(\sigma_{r} + f(U)) \\
& \geq \frac{1}{3\max(1, \alpha)} \left (1 -  e^{-1/\lambda } \right )f(\opt \cap V_1) + f(\opt \cap V_2) \\
& \geq \frac{1}{3\max(1, \alpha)} \left (1 -  e^{-1/\lambda } \right )f(\opt),
\end{align*}
where the last inequality follows from submodularity. The claim follows.
\subsection*{Proof of Theorem \ref{theorem:dynamic}}
We prove that the claim holds after $\ell > 1$ weights updates were given since the beginning of the optimization process. We denote with $W_i^* = \{W_{i,1},\dots, W_{i,k}\}$ the $i$-th new sequence of dynamic weights. Furthermore, define $\chi_{\ell} = \min \{ \absl{\chi(c, W_{\ell})}, \absl{\chi(c, W_{\ell - 1})}\}$ and $\mathcal{V}_\ell = \{e \in V \colon c_i(e) \leq \lambda W_{\ell, i}/k \ \forall i \in [k]\}$., let $\sigma_{\ell, t}$ be as the current solution at time step $t$, after new dynamic weights $W_\ell^*$ were given. Let $r_\ell \geq 0$ be the smallest index, such that
\begin{enumerate}
    \item $c_i(\sigma_{\ell, t}) \leq W_{\ell, i}$, for all $e \in V$, $t \in [r_{\ell} - 1]$ and for all $i \in [k]$;
    \item $c_i(\sigma_{\ell, r_{\ell}}) > W_{\ell, i}$, for some $i \in [k]$.
\end{enumerate}
In other words, $r_{\ell}$ is the first point it time, after the ${\ell}$-th weight update, such that the set maximizing the greedy step is not feasible. We prove that the \dgreedy{} maintains the desired approximation guarantee. Note that at each step, the \dgreedy{} requires at most $n$ calls to the value oracle function. Then the \kgreedy{} with restarts requires $\bigo{n (n - \chi_{\ell}))}$ additional run time to construct the solution $\sigma_{{\ell}, r_{\ell}}$.

In order to prove the desired lower-bound on the approximation guarantee, we prove that the solution $\sigma_{\ell, r_{\ell}}$ is identical to a solution of the same size constructed by the \kgreedy{} starting from the empty set, under side constraints specified by $W_{\ell}^*$. The claim then follows, by readily applying Theorem \ref{theorem:static}.

Define $d_i = \sigma_{\ell, r_{\ell}}$ and denote with $\{v_{{\ell}, 1}, \dots, v_{{\ell}, t}, \dots, v_{{\ell}, d_{\ell}}\}$ the points of $\sigma_{\ell, r_{\ell}}$ sorted in the order that they were added to the solution. Define the sets
\[
\begin{array}{lcl}
\hat{\sigma}_{{\ell}, t}  = \emptyset & \mbox{for } & t = 0;\\
\hat{\sigma}_{{\ell}, t}  = \{v_{{\ell}, 1}, \dots, v_{{\ell}, t}\} & \mbox{for } & t \in [d_{\ell}];\\
\end{array}
\]
Note that, according to this definition, it holds $\hat{\sigma}_{{\ell}, t} = \sigma_{{\ell}, r_{\ell}}$ for all ${\ell} \in [m]$. We prove that the solution $\sigma_{{\ell}, r_{\ell}}$ is equal to a solution if the same size constructed by the \kgreedy{} from scratch, by showing that it holds
\begin{equation}
\label{eq:dynamic_knapsack}
v_{{\ell}, t} = \argmax_{e \in V\setminus \hat{\sigma}_{{\ell}, t - 1}} \frac{f_{\hat{\sigma}_{{\ell}, t - 1}}(e)}{ \max_j c_j(e)}
\end{equation}
for all $t \in [d_{\ell}]$, for all ${\ell} \in [m]$, with an induction argument on ${\ell}$. The base case for ${\ell} = 1$ holds due to the greedy rule, since in this case the optimization process consists of greedily adding points to the current solution, starting from the empty set.

For the inductive case, suppose that the claim holds for all runs up to ${\ell} - 1$. Note that, since the functions $c_j$ are linear, it holds
\[
c_j(T) \leq \sum_{e \in T} \max_j c_j(e) \leq \sum_{e \in \chi_{\ell}} \max_j c_j(e) \leq W_{\ell, j},
\]
for all $T \in V$ such that $\absl{T} \leq \absl{\chi_{\ell}}$. Hence, all subsets of $V$ of size at most $T$ are feasible solutions in the $k$-knapsack intersection. Similarly, one can prove that $c_j(T) \leq W_{\ell - 1, j}$, for all $T \in V$ such that $\absl{T} \leq \absl{\chi_{\ell}}$. Hence, all solutions with size at most $\absl{\chi}$ are feasible in both constraints defined by $W_{\ell-1}^*$ and $W_{\ell}^*$. In particular, solutions obtained by Algorithm \ref{alg:static} up to size at most $\absl{\chi_{\ell}}$, are identical in both constraints, i.e. $\hat{\sigma}_{\ell, t} = \hat{\sigma}_{\ell -1, t}$. Combining these observations with the inductive hypothesis on $v_{{\ell} - 1, t}$, we get
\begin{equation*}
v_{{\ell}, t} = v_{{\ell} - 1, t} = \argmax_{e \in V\setminus \hat{\sigma}_{{\ell} - 1, t - 1}} \frac{{f}_{\hat{\sigma}_{{\ell} - 1, t - 1}}(e)}{\max_j c_j(e)} = \argmax_{e \in V\setminus \hat{\sigma}_{i, t - 1}} \frac{{f}_{\hat{\sigma}_{{\ell}, t - 1}}(e)}{\max_j c_j(e)},
\end{equation*}
for all $t \leq \absl{\chi_\ell}$. We conclude that \eqref{eq:dynamic_knapsack} holds for all $v_{i, t}$ with $t \leq  \absl{\chi_{\ell}}$. Note that for $\absl{\chi_{\ell}} < t \leq r_{\ell}$ the claim holds due to the greedy rule, hence \eqref{eq:dynamic_knapsack} holds. Furthermore, note that the new solution is in the set $\mathcal{V}_{\ell}$. Combining Theorem \ref{theorem:static} with \eqref{eq:dynamic_knapsack}, we conclude that $\sigma_{\ell, r_{\ell}}$ yields the desired approximation guarantee.
\end{document}